\documentclass{article}

\usepackage{microtype}
\usepackage{graphicx}
\usepackage{booktabs} 
\usepackage[margin=1in]{geometry}
\usepackage{amsmath, amssymb, amsthm, amsfonts}
\usepackage{tikz, , graphicx, comment, caption, subcaption}
\usepackage{algorithm}
\usepackage[noend]{algorithmic}
\usepackage{xr}

\newtheorem{theorem}{Theorem}
\newtheorem{example}{Example}
\newtheorem{observation}{Observation}
\newtheorem{corollary}{Corollary}
\newtheorem{claim}{Claim}
\newtheorem{lemma}{Lemma}
\newtheorem{definition}{Definition}

\newcommand {\M}{\mathcal{M}}
\newcommand {\F}{\mathcal{F}}
\newcommand {\N}{\mathcal {N}}

\usepackage{hyperref}

\begin{document}

\title{Proportionally Fair Clustering}
\author{Xingyu Chen \qquad Brandon Fain\thanks{corresponding author: btfain at cs.duke.edu} \qquad Liang Lyu \qquad Kamesh Munagala \\ Department of Computer Science, Duke University}
\date{}

\maketitle

\begin{abstract}We extend the fair machine learning literature by considering the problem of proportional centroid clustering in a metric context. For clustering $n$ points with $k$ centers, we define fairness as proportionality to mean that any $n/k$ points are entitled to form their own cluster if there is another center that is closer in distance for all $n/k$ points. We seek clustering solutions to which there are no such justified complaints from any subsets of agents, without assuming any a priori notion of protected subsets. We present and analyze algorithms to efficiently compute, optimize, and audit proportional solutions. We conclude with an empirical examination of the tradeoff between proportional solutions and the $k$-means objective. \end{abstract}

\section{Introduction}
\label{sec:introduction}
The data points in machine learning are often real human beings. There is legitimate concern that traditional machine learning algorithms that are blind to this fact may inadvertently exacerbate problems of bias and injustice in society~\cite{propublica}. Motivated by concerns ranging from the granting of bail in the legal system to the quality of recommender systems, researchers have devoted considerable effort to developing fair algorithms for the canonical supervised learning tasks of classification and regression~\cite{awareness, kleinberg2, opportunity, kleinberg1, mistreatment, threshold, calibration, groupEnvyFree, subgroup, convex, fairLossMinimization}. 

We extend this work to a canonical problem in unsupervised learning: centroid clustering. In centroid clustering, we want to partition data into $k$ clusters by choosing $k$ ``centers'' and then matching points to one of the centers. This is a classic context for clustering work~\cite{kcenter, STA97, kMedianLP, kmedians}, and is perhaps best known as the setting for the celebrated $k$-means heuristic (independently discovered many times, see~\cite{clusteringSurvey} for a brief history). We provide a novel group based notion of fairness as proportionality, inspired by recent related work on the fair allocation of public resources~\cite{justifiedRepresentation, fairPublicDecisions, coreEC18, marketsPublicDecisions}. We suppose that data points represent the individuals to whom we wish to be fair, and that these agents prefer to be clustered accurately (that is, they prefer their cluster center to be representative of their features). A solution is fair if it respects the entitlements of groups of agents, where we assume that a subset of agents is entitled to choose a center for themselves if they constitute a sufficiently large fraction of the population with respect to the total number of clusters (\textit{e.g.}, $1/k$ of the population, if we are clustering into $k$ groups). The guarantee must hold for \textit{all} subsets of sufficient size, and therefore does not hinge on any particular a priori knowledge about which points should be protected. This is in line with other recent observations that information about which individuals should be protected may not be available in practice~\cite{fairLossMinimization}.

Consider a motivating example where proportional clustering might be preferable to more standard clusterings that try to minimize the $k$-means or $k$-median objective. Suppose there are 3 spherical clusters in the data: A, B, and C, and we are computing a 3-clustering. A, B, and C each contain one third of the total data. The radius of A is very large compared to the radii of B and C, and A is very far away from B and C compared to the radius of A. The radii of B and C are very small, and B and C are close relative to the radius of A. More simply, A is a large sphere very far away from two small spheres B and C, which are close together.

Simply placing centers at the middle of A, B, and C is proportional. However, the global k-means or k-median minimizer places 1 center for B and C to share, and uses the remaining 2 centers to cover A. Such a solution is arbitrarily not-proportional as the radii of B and C become arbitrarily small. Essentially, the global optimum forces B and C to share a center in order to pay for the high variance in A. 

To interpret this example, suppose we are clustering home locations to decide where to build public parks. B and C are dense urban centers, and A is a suburb. Minimizing total distance seems reasonable, but the global optimum builds 2 parks for A, and only 1 that B and C must share. Alternatively, A, B, and C might represent clusters of patients in a medical study. Both solutions distinguish A from B and C, but the global optimum obscures the secondary difference between B and C. In both instances, B or C could represent a protected group (e.g., home location may be racially divided, and race or sex could cause differences in medical data), in which case proportionality provides a guarantee even if we do not have access to this information.

\subsection{Preliminaries and Definition of Proportionality} We have a set $\N$ of $|\N| = n$ individuals or data points, and a set $\M$ of $|\M| = m$ feasible cluster centers. We will sometimes consider the important special case where $\M = \N$ (i.e., where one is only given a single set of points as input), but most of our results are for the general case where we make no assumption about $\M \cap \N$. For all $i, j \in \N \cup \M$, we have a distance $d(i, j)$ satisfying the triangle inequality. Our task is centroid clustering as treated in the classic $k$-median, $k$-means, and $k$-center problems. We wish to open a set $X \subseteq \M$ of $|X| = k$ centers (assume $|\M| \geq k$), and then match all points in $\N$ to their closest center in $X$. For a particular solution $X$ and agent $i \in \N$, let $D_i(X) = \min_{x \in X} d(i, x)$. In general, a good clustering solution $X$ will have small values of $D_i(X)$, although the aforementioned objectives differ slightly in how they measure this. In particular, the $k$-median objective is $\sum_{i \in \N} D_i(X)$, the $k$-means objective is $\sum_{i \in \N} \left(D_i(X)\right)^2 $, and the $k$-center objective is $\max_{i \in \N} D_i(X)$.

To define proportional clustering, we assume that individuals prefer to be closer to their center in terms of distance (i.e., ensuring that the center is more representative of the point). Any subset of at least $r \lceil \frac{n}{k} \rceil$ individuals is entitled to choose $r$ centers. We call a solution proportional if there does not exist any such sufficiently large set of individuals who, using the number of centers to which they are entitled, could produce a clustering among themselves that is to their mutual benefit in the sense of Pareto dominance. More formally, a $\textit{blocking coalition}$ is a set $S \subseteq \N$ of at least $r \lceil \frac{n}{k} \rceil$ points and a set $Y \subseteq \M$ of at most $r$ centers such that $D_i(Y) < D_i(X)$ for all $i \in S$. It is easy to see that because $D_i(X) = \min_{x \in X} d(i, x)$, this is functionally equivalent to Definition~\ref{def:core}; a larger blocking coalition necessarily implies a blocking coalition with a single center. 

\begin{definition}
\label{def:core}
	Let $X \subseteq \M$ with $|X| = k$. $S \subseteq N$ is a \textbf{blocking coalition} against $X$ if $|S| \geq \lceil \frac{n}{k} \rceil$ and $\exists y \in \M$ such that $\forall i \in S$, $d(i, y) < D_i(X)$. 	$X \subseteq \N$ is \textbf{proportional} if there is no blocking coalition against $X$. 
\end{definition}

Equivalently, 
$X$ is proportional if $\forall S \subseteq \N$ with $|S| \geq \lceil \frac{n}{k} \rceil$ and for all $y \in \M$, there exists $i \in S$ with $d(i, y) \geq D_i(X)$. It is important to note that this quantification is over \textit{all} subsets of sufficient size. Hence, in attempting to satisfy the guarantee for a particular subset $S$, one cannot simply consider a single $i \in S$ and ignore all of the other points, as $S \backslash \{i\}$ may itself be a subset to which the guarantee applies.

It is instructive to briefly consider an example. In Figure~\ref{fig:coreExample}, $\N = \M$, $k = 2$, and there are $12$ individuals, represented by the embedded points. Suppose we want to minimize the $k$-center objective in the pursuit of fairness: we would then choose the red points. However, this is not a proportional solution because the middle six points constitute half of the points, and would all prefer to be matched to the central blue point. Furthermore, choosing the blue point (and any other center) \textit{is} a proportional solution, because for any arbitrary group of six points and new proposed center, at least one of the six points will be closer to the blue point than the proposed center. 

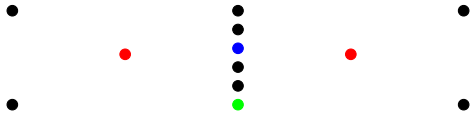
\begin{figure}[h]
\centering
\begin{tikzpicture}
	\filldraw (-3,-0) circle (2pt);
	\filldraw (-3,1.25) circle (2pt);
	\filldraw [green] (0,0) circle (2pt);
	\filldraw (0,0.25) circle (2pt);
	\filldraw (0,0.5) circle (2pt);
	\filldraw [blue] (0,0.75) circle (2pt);
	\filldraw (0,1) circle (2pt);
	\filldraw (0,1.25) circle (2pt);
	\filldraw (3,0) circle (2pt);
	\filldraw (3,1.25) circle (2pt);
	\filldraw [red] (-1.5,0.67) circle (2pt);
	\filldraw [red] (1.5,0.67) circle (2pt);
\end{tikzpicture}
\caption{Proportionality Example}
\label{fig:coreExample}
\end{figure} 

Proportionality has many advantages as a notion of fairness in clustering, beyond the intuitive appeal of groups being entitled to a proportional share of centers. We name a few of these advantages explicitly.
\begin{itemize}
	\item Proportionality implies (weak) \textit{Pareto optimality}: namely, for any proportional solution $X$, there does not exist another solution $X'$ such that $D_i(X') < D_i(X)$ for all $i \in \N$.
	\item Proportionality is \textit{oblivious} in the sense that it does not depend on the definition of sensitive attributes or protected sub-groups.
	\item Proportionality is \textit{robust} to outliers in the data, since only groups of points of sufficient size are entitled to their own center. 
	\item Proportionality is \textit{scale invariant} in the sense that a multiplicative scaling of all distances does not affect the set of proportional solutions.
	\item Approximately proportional solutions can be \textit{efficiently computed}, and one can optimize a secondary objective like $k$-median subject to proportionality \textit{as a constraint}, as we show in Section~\ref{sec:theory} and Section~\ref{sec:lp}.
	\item Proportionality can be \textit{efficiently audited}, in the sense that one does not need to compute the entire pairwise distance matrix in order to check for violations of proportionality, as we show in Section~\ref{sec:audit}.
\end{itemize}

In the worst case, proportionality is incompatible with all three of the classic $k$-center, $k$-means, and $k$-median objectives; {\em i.e.}, there exist instances for which any proportional solution has an arbitrarily bad approximation to all objectives.  We present such an instance in Example~\ref{example:impossibility}, and show in Section~\ref{sec:experiments} that this behavior also arises in real-world datasets.  

\begin{example} 
\label{example:impossibility}
There exist problems for which any proportional clustering has an unbounded approximation to the optimal $k$-center, $k$-means, and $k$-median objectives, and conversely, any clustering with bounded approximation to the optimum on these objectives is not proportional.  
\end{example}
\begin{proof}
The simplest example to see this has $\N = \M$, $n=6$, and $k=3$ (i.e., we want to choose 3 centers from six individuals, all of which are possible cluster centers). There are two points at position a, two at position b, and one each at positions c and d. The pairwise distances are given in the following matrix.
\begin{center}
  \begin{tabular}{| c || c | c | c | c |} 
  \hline
	& a & b & c & d \\ \hline \hline
	a & 0 & 1 & $\infty$ & $\infty$ \\ \hline
	b & 1 & 0 & $\infty$ & $\infty$ \\ \hline
	c & $\infty$ & $\infty$ & 0 & $\infty$ \\ \hline
	d & $\infty$ & $\infty$ & $\infty$ & 0 \\ \hline
\end{tabular}
\end{center}
Because $n=6$ and $k$=3, and there are $6/3 = 2$ points at $a$ and $b$, any proportional solution must include $a$ and $b$. Therefore, one of the points at $c$ or $d$ will have an arbitrarily large value $D_i(X)$. The optimal solution on any of the three objectives is to instead choose $c,d,$ and one of $a$ or $b$.
\end{proof}

Furthermore, as we show in Section~\ref{sec:theory} and observe empirically in Section~\ref{sec:experiments}, proportional solutions may not always exist. We therefore consider the natural approximate notion of proportionality that relaxes the Pareto dominance condition by a multiplicative factor.

\begin{definition}
	$X \subseteq \M$ with $|X| = k$ is $\rho$-\textbf{approximate proportional} (hereafter $\rho$-proportional) if $\forall S \subseteq \N$ with $|S| \geq \lceil \frac{n}{k} \rceil$ and for all $y \in \M$, there exists $i \in S$ with $\rho \cdot d(i, y) \geq D_i(X)$.
\end{definition}

To parse the definition, again consider Figure~\ref{fig:coreExample}. Although choosing the red points is not a proportional solution, it is an approximate proportional solution. To see this, suppose the middle six agents wish to deviate to the blue point as before. The green agent would decrease it's distance to a center by deviating, but not by more than a constant factor, say $3$, so the red points would constitute a 3-proportional solution. Note that even with this notion, it remains true that any approximately proportional clustering is incompatible with any approximately optimal clustering on the $k$-center, $k$-means, and $k$-median objectives, in the worst case.


\subsection{Results and Outline} 
In Section~\ref{sec:theory} we show that proportional solutions may not always exist. In fact, one cannot get better than a $2$-proportional solution in the worst case. In contrast, we give a greedy algorithm (Algorithm~\ref{alg:ball}) and prove Theorem~\ref{thm:ballGrowing}: The algorithm yields a $\left( 1 + \sqrt{2} \right)$-proportional solution in the worst case. 

In Section~\ref{sec:lp}, we treat proportionality as a constraint and seek to optimize the $k$-median objective subject to that constraint. We show how to write approximate proportionality as $m$ linear constraints. Incorporating this into the standard linear programming relaxation of the $k$-median problem, we show how to use the rounding from~\cite{kMedianLP} to find an $O(1)$-proportional solution that is an $O(1)$-approximation to the $k$-median objective of the optimal proportional solution.

In Section~\ref{sec:audit}, we show that proportionality is approximately preserved if we take a random sample of the data points of size $\tilde{O}(k^3)$, where the $\tilde{O}$ hides low order terms. This immediately implies that for constant $k$, we can check if a given clustering is proportional as well as compute approximately proportional solutions in near linear time, comparable to the time taken to run the classic $k$-means heuristic.

In Section~\ref{sec:experiments}, we provide a local search heuristic that efficiently searches for a proportional clustering. Our heuristic is able to consistently find nearly proportional solutions in practice. We test our heuristic and Algorithm~\ref{alg:ball} empirically against the celebrated $k$-means heuristic in order to understand the tradeoff between proportionality and the $k$-means objective. We find that the tradeoff is highly data dependent: Though these objectives are compatible on some datasets, there exist others on which these objectives are in conflict. 



\subsection{Related Work}
\textbf{Unsupervised Learning.} Metric clustering is a well studied problem. There are constant approximation polynomial time algorithms for both the $k$-median~\cite{kMedianPrimalDual, kMedianLP, kmedians, mettuPlaxton, BPRST17} and $k$-center objective~\cite{kcenter, STA97}.  Proportionality is a constraint on the {\em centers} as opposed to the data points; this makes it difficult to adapt standard algorithmic approaches for $k$-medians and $k$-means such as local search~\cite{kmedians}, primal-dual~\cite{kMedianPrimalDual}, and greedy dual fitting~\cite{kMedianApprox}. For instance, our greedy algorithm in Section~\ref{sec:theory} grows balls around potential centers, which is very different from how balls are grown in the primal-dual schema~\cite{kMedianPrimalDual, mettuPlaxton}. Somewhat surprisingly, in Section~\ref{sec:theory} we show that for the problem of minimizing the $k$-median objective subject to proportionality as a constraint, we can extend the linear program rounding technique of~\cite{kMedianLP} to get a constant approximation algorithm. However, the additional constraints we add in the linear program formulation render the primal-dual and other methods inapplicable.

In~\cite{fairlets} and subsequent generalizations~\cite{RS45, BCN19}, the authors consider fair clustering in terms of balance: There are red and blue points, and a balanced solution has roughly the same ratio of blue to red points in every cluster as in the overall population. The authors are motivated to extract features that cannot discriminate between status in different groups. This ensures that subsequent regression or classification on these features will be fair between these groups. In contrast, we assume that our data points prefer to be accurately clustered, and that an unfair solution provides accurate clusters for some groups, while giving other large groups low quality clusters. Finally, we note that there is a line of work in fair unsupervised learning concerned with constructing word embeddings that avoid bias~\cite{manWoman, biasScience}, but these problems seem orthogonal to our concerns in clustering.

\textbf{Supervised Learning.} The standard model in fair supervised learning~\cite{awareness, kleinberg2, kleinberg1, groupEnvyFree, mistreatment} has a set of \textit{protected agents} given as input to an algorithm which must classify agents into a positive and negative group. Most of these notions of fairness do not apply in any natural way to unsupervised learning problems. Our work further differs from the supervised learning literature in that we do not assume information about which agents are to be protected. Instead, we provide a fairness guarantee to arbitrary groups of agents, including protected groups even if we do not know their identity, similar to the ideas considered in~\cite{subgroup} and~\cite{fairLossMinimization}.

\textbf{Fair Resource Allocation.} Our notion of proportionality is derived from the notion of core in economics~\cite{scarfCore, lindahlCore}. The core has been adapted as a natural generalization to groups of the idea of fairness as proportionality~\cite{coreWINE16, coreEC18}, similar to other group fairness concepts for public goods that explicitly consider shared resources~\cite{fairPublicDecisions, justifiedRepresentation}. In clustering, the public goods are the centers themselves, and the ``agents'' are the data points, which share the centers. The fair clustering problem differs in that it is framed in terms of costs instead of positive utility, and agents only care about their most preferred good. That is, an agent's cost for a clustering solution is just the distance to the closest center, as opposed to much of the previous resource allocation literature where agents have additive utility across the allocated goods. One can interpret our work as results for computing the core for a resource allocation problem where agents have a min-cost function with respect to allocations.


\section{Existence and Computation of Proportional Solutions} 
\label{sec:theory}
We begin with a negative result: in the worst case, there may not be an exact proportional solution. The impossibility remains even in the special case when $\N = \M$. The latter construction is slightly more involved; we begin with the arbitrary $\N$ and $\M$ setting. The basic idea behind both constructions is to create two groups of points very far away from one another with $k=3$, ensuring that one group will be served by only one center. 

\begin{claim}
\label{claim:lowerBound}
	For all $\rho < 2$, a $\rho$-proportional solution is not guaranteed to exist.
\end{claim}
\begin{proof}
	Consider the following instance with $\N = \{a_1,  a_2, \hdots, a_6\}$, $\M = \{x_1, x_2, \hdots, x_6\}$ and $k=3$. Distances are specified in the following table.
\begin{center}
  \begin{tabular}{|c || c | c | c | c | c | c  |} 
  \hline 
         & $x_1$   & $x_2$  & $x_3$    & $x_4$  & $x_5$   & $x_6$   \\ \hline \hline
  $a_1$ &  4     & 1 & 2    & $\infty$ & $\infty$ & $\infty$  \\ \hline
  $a_2$ & 2    &   4    & 1  & $\infty$ & $\infty$ & $\infty$  \\ \hline  
  $a_3$ & 1  &  2  &   4     & $\infty$ & $\infty$ & $\infty$  \\ \hline
  $a_4$ & $\infty$ & $\infty$ & $\infty$ &    4    & 1 &  2    \\ \hline
  $a_5$ & $\infty$ & $\infty$ & $\infty$ &   2  &    4   & 1   \\ \hline
  $a_6$ & $\infty$ & $\infty$ & $\infty$ & 1  &  2  &   4       \\ \hline
\end{tabular}
\end{center}

Notice that the data is separate into two areas. Since $k=3$, in a feasible solution, we only open one center in one of these two areas. Without loss of generality, suppose that we open exactly one center among $\{x_1, x_2, x_3\}$. The instance is symmetric, so again suppose without loss of generality that we open $x_1$. Then consider the individuals in $\{a_1, a_2\}$. This coalition is of size $\lceil \frac{n}{k} \rceil = 2$, and both individuals would reduce their distance by a factor of $2$ by switching to $x_3$. Thus, any solution is only 2-proportional. 
\end{proof}

\begin{claim}
	In the special case where $\N = \M$, for all $\rho < 1.5$, a $\rho$-proportional solution is not guaranteed to exist.
\end{claim}
\begin{proof}
	Let $k=5$. There are three identical clusters of 303 points each (so $n = 909$). The pairwise distance between any two points in different clusters is $\infty$. Construct each cluster as follows. There are six types of points, $a_1, a_2, a_3, x_1, x_2, x_3$ (all feasible, since $\M = \N$). There is exactly one point of type $x_1$, one point of type $x_2$, and one point of type $x_3$. There are 100 points each of type $a_1$, $a_2$, and $a_3$ (that is, there are 100 points co-located at each position). The pairwise distance between points in a cluster of given types is specified in the following table. The pairwise distance between any two points of types in $\{x_1, x_2, x_3\}$ is equal to the distance between any two points of types in $\{a_1, a_2, a_3\}$ is equal to 3, which follows from the shortest path distances on a weighted bipartite graph with weights defined by the table. 
		
\begin{center}
  \begin{tabular}{|c || c | c | c |} 
  \hline 
         & $x_1$ & $x_2$ & $x_3$ \\ \hline \hline
  $a_1$ & 4     &  1        & 2         \\ \hline
  $a_2$ & 2     &  4        & 1         \\ \hline  
  $a_3$ & 1     &  2        & 4         \\ \hline
\end{tabular}
\end{center}

Since $k = 5$ and there are three clusters, in a feasible solution there is a cluster in which we choose at most one center. In that cluster, suppose first that we choose a center of type $x_1$, $x_2$, or $x_3$. Since the instance is symmetric with respect to this choice, suppose without loss of generality that we choose $x_1$. Then the 200 points of types $a_1$ and $a_2$ could decrease their distance by a factor of $2$ by switching to $x_3$. Since any $\lceil 909/5 \rceil = 182$ points are entitled to deviate, such a choice of $x_1$ is not $\rho$-proportional for $\rho < 2$. 

Instead, suppose that we choose a center of type $a_1$, $a_2$, or $a_3$. The instance is again symmetric with respect to this choice, so suppose without loss of generality that we choose $a_1$. Then the 200 points of types $a_2$ and $a_3$ could decrease their distance by a factor of 1.5 by switching to $x_1$. Thus, in either case, the solution is not $\rho$-proportional for any $\rho < 1.5$. 
\end{proof}

\subsection{Computing a $ \left(1+\sqrt{2} \right)$-Approximate Proportional Clustering}
Claim~\ref{claim:lowerBound} establishes that we should focus our attention on designing an efficient approximation algorithm. We give a simple and efficient algorithm that achieves a $\left( 1 + \sqrt{2} \right)$-proportional solution, very close to the existential lower bound of 2. For notational ease, let $B(x, \delta) = \{i \in \N: \; d(i, x) \leq \delta\}$. That is, $B(x, \delta)$ is the ball (defined on $\N$) of distance $\delta$ about center $x$. For simplicity of exposition, we present Algorithm~\ref{alg:ball} as a \textit{continuous} algorithm where a $\delta$ parameter is smoothly increasing. The algorithm can be easily discretized using priority queues. 

\begin{algorithm}
\begin{algorithmic}[1]
\caption{Greedy Capture}
\label{alg:ball}
	\STATE $\delta \gets 0$; $X \gets \emptyset$; $N \gets \N$
	\WHILE{$N \neq \emptyset$}
		\STATE Smoothly increase $\delta$
		\WHILE{$\exists x \in X$ s.t. $|B(x, \delta) \cap N| \geq 1$}
			\STATE $N \gets N \backslash B(x, \delta)$
		\ENDWHILE
		\WHILE{$\exists x \in ( \M \backslash X)$ s.t. $|B(x, \delta) \cap N| \geq \lceil \frac{n}{k} \rceil$}
			\STATE $X \gets X \cup \{x\}$
			\STATE $N \gets N \backslash B(x, \delta)$
		\ENDWHILE
	\ENDWHILE
	\STATE return X
\end{algorithmic}
\end{algorithm}

Algorithm~\ref{alg:ball} runs in $\tilde{O}(mn)$ time.\footnote{To state running times simply, we use the convention that $f(n)$ is $\tilde{O}(g(n))$ if $f(n)$ is $O(g(n))$ up to poly-logarithmic factors.} In essence, the algorithm grows balls continuously around the centers, and when the ball around a center has ``captured'' $\lceil \frac{n}{k} \rceil$ points, we greedily open that center and disregard all of the captured points. Open centers continue to greedily capture points as their balls continue to expand. Though~\cite{kMedianPrimalDual,mettuPlaxton} similarly expand balls about points to compute approximately optimal solutions to the $k$-median problem, there is a crucial difference: They grow balls around data points rather than centers. 

\begin{theorem}
\label{thm:ballGrowing}
	Algorithm~\ref{alg:ball} yields a $(1+\sqrt{2})$-proportional clustering, and there exists an instance for which this bound is tight.
\end{theorem}
\begin{proof}
	Let $X$ be the solution computed by Algorithm~\ref{alg:ball}. First note that $X$ uses at most $k$ centers, since it only opens a center when $\lceil \frac{n}{k} \rceil$ \textit{unmatched} points are absorbed by the ball around that center, and this can happen at most $k$ times. Now, suppose for a contradiction that $X$ is not a $(1+\sqrt{2})$-proportional clustering. Then there exists $S \subseteq \N$ with $|S| \geq \lceil \frac{n}{k} \rceil$ and $y \in \M$ such that \begin{equation} \label{eq:contradiction} \forall i \in S, \; (1+\sqrt{2}) \cdot d(i, y) < D_i(X).\end{equation} 

Let $r_y$ be the distance of the farthest agent from $y$ in $S$, that is, $r_y := \max_{i \in S} d(i, y)$, and call this agent $i^*$. There are two cases. In the first case, $B(x, r_y) \cap S = \emptyset$ for all $x \in X$. This immediately yields a contradiction, because it implies that Algorithm 1 would have opened $y$. In particular, note that $S \subseteq B(y, r_y)$, so if $S \cap B(x, r_y) = \emptyset$ for all $x \in X$, then $B(y, r_y)$ would have had at least $\lceil \frac{n}{k} \rceil$ unmatched points.

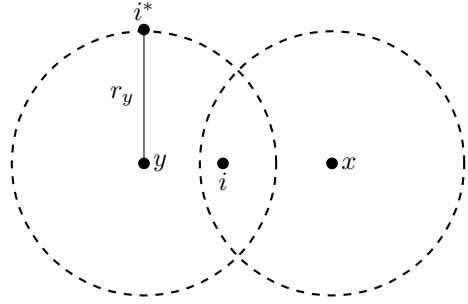
\begin{figure}[!h]
\centering
\begin{tikzpicture}
	\filldraw[black] (-1.25,0) circle (2pt) node[anchor=west] {$y$};
	\draw[black, thick, dashed] (-1.25,0) circle (50pt);
	
	\filldraw[black] (1.25,0) circle (2pt) node[anchor=west] {$x$};
	\draw[black, thick, dashed] (1.25,0) circle (50pt);
	
	\filldraw[black] (-1.25,1.78) circle (2pt) node[anchor=south] {$i^*$};
	\draw (-1.25,0) -- (-1.25, 1.78);
	\filldraw[black] (-1.25,0.89) circle (0pt) node[anchor=east] {$r_y$};
	
	\filldraw[black] (-0.2,0) circle (2pt) node[anchor=north] {$i$};
	
	\end{tikzpicture}
\caption{Diagram for Proof of Theorem~\ref{thm:ballGrowing}}
\label{fig:ballProof}
\end{figure}

In the second case, $\exists x \in X$ and $\exists i \in N$ such that $i \in B(x, r_y) \cap S$. This case is drawn below in Figure~\ref{fig:ballProof}. By the triangle inequality, $d(x, y) \leq d(i, x) + d(i, y)$. Therefore, $d(i^*, x) \leq r_y + d(i, x) + d(i, y)$. Also, $d(i, x) \leq r_y$, since $i \in B(x, r_y)$. Consider the minimum multiplicative improvement of $i$ and $i^*$:

\begin{equation*}
\begin{split}
	&\min\left( \frac{d(i, x)}{d(i, y)}, \; \frac{d(i^*, x)}{d(i^*, y)} \right) \\
	& \leq \min\left( \frac{d(i, x)}{d(i, y)}, \; \frac{r_y + d(i, x) + d(i, y)}{r_y} \right) \\
	& \leq \min\left( \frac{r_y}{d(i, y)}, \; 2 + \frac{d(i, y)}{r_y} \right) \\
	& \leq \max_{z \geq 0} \left( \min \left( z,\; 2+1/z\right) \right) = 1 + \sqrt{2} 
\end{split}
\end{equation*}
which violates equation~\ref{eq:contradiction}. 

It is not hard to show that there exists an instance for which Algorithm~\ref{alg:ball} yields exactly this bound. Consider the following instance with $\N = \{a_1,  a_2, \hdots, a_6\}$, $\M = \{x_1, x_2, x_3, x_4\}$ and $k=3$. Distances are specified in the following table, where $\epsilon > 0$ is some small constant.
\begin{center}
  \begin{tabular}{|c || c | c | c | c | c | c  |} 
  \hline 
             & $x_1$          & $x_2$       & $x_3$          & $x_4$             \\ \hline \hline
  $a_1$ & 1                  & $1+\sqrt{2}$ & $\infty$        & $\infty$         \\ \hline
  $a_2$ & $\sqrt{2}-1$ &  $1-\epsilon$ & $\infty$        & $\infty$         \\ \hline  
  $a_3$ & $1+\sqrt{2}$ &  $1-\epsilon$ & $\infty$        & $\infty$          \\ \hline
  $a_4$ & $\infty$         & $\infty$     & 1                  &   $1+\sqrt{2}$ \\ \hline
  $a_5$ & $\infty$         & $\infty$     & $\sqrt{2}-1$  &   $1-\epsilon$ \\ \hline
  $a_6$ & $\infty$         & $\infty$     & $1+\sqrt{2}$ & $1-\epsilon$  \\ \hline
\end{tabular}
\end{center}
The distances satisfy the triangle inequality. Note that Algorithm~\ref{alg:ball} will open $x_2$ and $x_4$. The coalition $\{a_1, a_2\}$ can each reduce their distance by a multiplicative factor approaching $1+\sqrt{2}$ as $\epsilon \rightarrow 0$ by deviating to $x_1$.
\end{proof}

\subsection{Local Capture Heuristic}
We observe that while our Greedy Capture algorithm (Algorithm~\ref{alg:ball}) always produces an approximately proportional solution, it may not produce an exactly proportional solution in practice, even on instances where such solutions exist (see Figure~\ref{fig:IrisRho} and Figure~\ref{fig:DiabetesRho}). We therefore introduce a Local Capture heuristic for searching for more proportional clusterings. Algorithm~\ref{alg:local} takes a target value of $\rho$ as a parameter, and proceeds by iteratively finding a center that violates $\rho$-fairness and swapping it for the center in the current solution that is least demanded.

\begin{algorithm}
\begin{algorithmic}[1]
\caption{Local Capture Heuristic}
\label{alg:local}
	\INPUT $\rho$
	\STATE Initialize $X$ as a random subset of $k$ centers from $\M$.
	\REPEAT
	\FOR{$y \in \M$}
		\STATE $S_y \gets \{i \in N: \rho \cdot d_{iy} < D_i(X)\}$
		\IF{$|S_y| \geq \lceil \frac{n}{k} \rceil$} 
			\STATE $x^* \gets \mbox{argmin}_{x \in X} | \{i \in N: d_{ix} = D_i(X) \}|$
			\STATE $X \gets (X \backslash \{x^*\}) \cup \{y\}$
		\ENDIF
	\ENDFOR
	\UNTIL{no changes occur}
	\STATE return X
\end{algorithmic}
\end{algorithm}

Every iteration of Algorithm~\ref{alg:local} (the entire inner for loop) runs in $\tilde{O}(mn^2)$ time. There is no guarantee of convergence (for a given input $\rho$, there may not even exist a $\rho$-proportional solution), but if Algorithm~\ref{alg:local} terminates, then it returns a $\rho$-proportional solution. In our experiments (see Section~\ref{sec:experiments}), we search for the minimum $\rho$ for which the algorithm terminates in a small number of iterations via binary search over possible input of $\rho$. In~\cite{kmedians}, the authors also evaluate a local search swapping procedure for the $k$-median problem, but their swap condition is based on the relative $k$-median objective of two solutions, whereas our swap condition is based on violations to proportionality.

\section{Proportionality as a Constraint}
\label{sec:lp}
One concern with the previous algorithms is that they may find a  proportional clustering with poor global objective (e.g., $k$-median), even when exact proportional clusterings with good global objectives exist. For example, suppose $k=2$ and there are two easily defined clusters, containing $40\%$ and $60\%$ of the data respectively. It is possible that Algorithm~\ref{alg:ball} will only open centers inside of the larger cluster. This is proportional, but undesirable from an optimization perspective (note that the ``correct'' clustering of such an example is still proportional). Here, we show how to address this concern by optimizing the $k$-median objective subject to proportionality as a constraint. Later, in Section~\ref{sec:experiments}, we empirically study the tradeoff between the $k$-means objective and proportionality on real data.

We consider the $k$-median and $k$-means objectives to be reasonable measures of the global quality of a solution. We see minimizing the $k$-center objective more as a competing notion of fairness, and so we focus on optimizing the $k$-median objective subject to proportionality.\footnote{A constant approximation algorithm for minimizing the $k$-median objective immediately implies a constant approximation algorithm for minimizing the $k$-means objective by running the algorithm on the squared distances~\cite{mettuPlaxton}.} 
Minimizing the $k$-median objective without proportionality is a well studied problem in approximation algorithms, and several constant approximations are known~\cite{kMedianLP, kmedians, mettuPlaxton}. 
Most of this work is in the model where $\N \subseteq \M$, and we follow suit in this section. 
We show the following.

\begin{theorem}
\label{thm:lp}
Suppose there is a $\rho$-proportional clustering with $k$-median objective $c$. In polynomial time in $m$ and $n$, we can compute a $O(\rho)$-proportional clustering with $k$-median objective at most $8 c$.
\end{theorem}

In particular, we can compute a constant approximate proportional clustering with $k$-median objective at most eight times the minimum $k$-median objective proportional clustering. Note that the exact running time will depend on the algorithm used to solve the linear program. In the remainder of this section, we will sketch the proof of Theorem~\ref{thm:lp}. We begin with the standard linear programming relaxation of the $k$-median minimization problem, and then add a constraint to encode proportionality. The final linear program is shown in Figure~\ref{fig:lp}. Recall that $B(x, \delta) = \{i \in \N: \; d(i, x) \leq \delta\}$.

\begin{figure}
\begin{align}
	\label{obj} & \mbox{Minimize} & \sum_{i \in \N} \sum_{j \in \M} d(i, j) z_{ij} && \\
	\label{con1} & \mbox{Subject to} & \sum_{j \in \M} z_{ij} = 1 &  & \forall i \in \N \\
	\label{con2} && z_{ij} \leq y_j  &  & \forall j \in \M, \forall i \in \N \\
	\label{con3} && \sum_{j \in \M} y_j \leq k && \\
	\label{con4} && \sum_{j' \in B(j, \gamma R_j)} y_{j'} \geq 1 & & \forall j \in \M \\
	\label{con5} && z_{ij}, y_j \in [0, 1] & & \forall j \in \M, \forall i \in \N
\end{align} 
\caption{Proportional $k$-median Linear Program}
\label{fig:lp}
\end{figure}

In the LP, $z_{ij}$ is an indicator variable equal to 1 if $i \in \N$ is matched to $j \in \M$. $y_j$ is an indicator variable equal to $1$ if $j \in X$, i.e., if we want to use center $j \in \M$ in our clustering. Objective~\ref{obj} is the $k$-median objective. Constraint~\ref{con1} requires that every point be matched, and constraint~\ref{con2} only allows a point to be matched to an open center. Constraint~\ref{con3} allows at most $k$ centers to be opened, and constraint~\ref{con5} relaxes the indicator variables to real values between 0 and 1. 

Constraint~\ref{con4} is the new constraint that we introduce. Our crucial lemma argues that constraint~\ref{con4} approximately encodes proportionality. Let $R_j$ be the minimum value such that $|B(j, R_j)| \geq \lceil \frac{n}{k} \rceil$. In other words, $R_j$ is the distance of the $\lceil \frac{n}{k} \rceil$ farthest point in $\N$ from $j$.

\begin{lemma}
\label{lemma:fairnessConstraint}
	Let $X$ be a clustering, and let $\gamma \geq 1$. If \,$\forall j \in \M$ there exists some $x \in X$ such that $d(j, x) \leq \gamma R_j$, then $X$ is $(1+\gamma)$-proportional. If $X$ is $\gamma$-proportional, then $\forall j \in \M$ there exists some $x \in X$ such that $d(j, x) \leq (1+\gamma) R_j$.
\end{lemma}
\begin{proof}
	Suppose that $\forall j \in \M$ there exists some $x \in X$ such that $d(j, x) \leq \gamma R_j$. Suppose for a contradiction that $X$ is not $(1+\gamma)$-proportional. Then there exists $S \subseteq \N$ with $|S| \geq \lceil \frac{n}{k} \rceil$ and $j \in \M$ such that $\forall i \in S, \; (1+\gamma) \cdot d(i, j) < D_i(X)$. By assumption, $\exists x \in X$ such that $d(j, x) \leq \gamma R_j$, so by the triangle inequality $D_i(X) \leq d(i, j) + d(j, x) \leq d(i, j) + \gamma R_j$. Therefore, $\forall i \in S, \; \gamma \cdot d(i, j) < D_i(X) - d(i, j) \leq \gamma R_j$. However, by definition of $R_j$, since $|S| = \lceil \frac{n}{k} \rceil$, there must exist some $i \in S$ such that $d(i, j) \geq R_j$.
	
	Suppose that $X$ is $\gamma$-proportional. Let $j \in \M$. Consider the set $S$ of the closest $\lceil \frac{n}{k} \rceil$ points in $\N$ to $j$. By definition of proportionality $\exists i \in S$ and $x \in X$ such that $\gamma d(i, j) \geq d(i, x)$. Therefore, by the triangle inequality, $d(j, x) \leq d(i, j) + d(i, x) \leq (1+\gamma) d(i, j)$. By definition of $S$, $d(i, j) \leq R_j$, so there exists $x \in X$ such that $d(j, x) \leq (1+\gamma) R_j$.  
\end{proof}

Now, suppose there is a $\rho$-proportional clustering $X$ with $k$-median objective $c$. Then we write the linear program shown in Figure~\ref{fig:lp} with $\gamma = \rho+1$ in constraint~\ref{con4}. Lemma~\ref{lemma:fairnessConstraint} guarantees that $X$ is feasible for the resulting linear program, so the fractional solution has $k$-median objective at most $c$.  We then round the resulting fractional solution. In \cite{kMedianLP}, the authors give a rounding algorithm for the the linear program in Figure~\ref{fig:lp} without Constraint~\ref{con4}. 
We show that a slight modification to this rounding algorithm also 
preserves Constraint~\ref{con4} to a constant approximation.

\begin{lemma} (Proved in Section~\ref{sec:lemmaProof})
\label{lemma:rounding}
Let $\{y_j\}, \{z_{ij}\}$ be a fractional solution to the linear program in Figure~\ref{fig:lp}. Then there is an integer solution $\{\hat{y_j}\}, \{ \hat{z_{ij}}\}$ that is an $8$-approximation to the objective, and that opens $k$ centers. Furthermore, for all $j \in \M$, $\sum_{j' \in B(j, \, 27 \gamma R_j)} \hat{y_{j'}} \geq 1$.
\end{lemma}

The proof of Lemma~\ref{lemma:rounding} is in parts and involves the technical details of~\cite{kMedianLP}. We provide the proof in section~\ref{sec:lemmaProof}, but first we complete the overall proof of Theorem~\ref{thm:lp}. Given Lemma~\ref{lemma:rounding}, applying Lemma~\ref{lemma:fairnessConstraint} again implies that the result of the rounding is $(27(1+\rho)+1)$-proportional, since we set $\gamma = 1+\rho$. Since the $k$-median objective of the fractional solution is at most $c$, the fact that the $k$-median objective of the rounded solution is at most $8c$ follows directly from the proof from~\cite{kMedianLP}. We note that the constant factor of 27 can be improved to 13 in the special case where $\N = \M$. Interestingly, the ostensibly similar primal-dual approach of~\cite{kMedianPrimalDual} does not appear amenable to the added constraint of proportionality (in particular, the reduction to facility location from~\cite{kMedianPrimalDual} is no longer straightforward).    


\subsection{Proof of Lemma~\ref{lemma:rounding}}
\label{sec:lemmaProof}

First, we give a brief overview of the method from \cite{kMedianLP}. Note that our goal in this argument is to show that the new constraint we added (constraint~\ref{con4}) is approximately satisfied after this rounding. The authors work in a setting where points can have \textit{demand}. In our setting, this just corresponds to points in $\N$ having a demand of 1, and moving or consolidating demand can be thought of as changing the instance by moving points in $\N$. Note that the original linear program requires $\N \subseteq \M$, and we follow suit in this proof. 

Given a fractional solution to the linear program in Figure~\ref{fig:lp} as $\{y_j\}, \{z_{ij}\}$, let $\bar{C_i} = \sum_{j \in \M} d(i, j) z_{ij}$. That is, $\bar{C_i}$ is the contribution of point $i$ to the $k$-median objective in the fractional optimum.
\begin{itemize}
  \item Step 1: Consolidate all demands $t_i$ to obtain $\{t'_i\}$ such that for all $i, j \in \M$ with $t'_i > 0, t'_j > 0$, they must be sufficiently far away such that $c_{ij}>4\max(\bar{C_i},\bar{C_j})$. Let $\M'$ be the set of centers with positive demand after this step, i.e. $\M' = \{j \in \M: t'_j > 0\}$.
  \item Step 2a: Consolidate open centers by moving each center not in $\M'$ to the nearest center in $\M'$. This gives a new solution $\{y'_j\}, \{z'_{ij}\}$ with $y'_j \geq \frac{1}{2}$ if $j \in \M'$ and $y'_j = 0$ if $j \notin \M'$. We call this a $\frac{1}{2}$-restricted solution.
  \item Step 2b: Modify the solution further to obtain $\{ \bar{y_j} \}, \{ \bar{z_{ij}} \}$ with $\bar{y_j} \in \{ \frac{1}{2}, 1 \}$ if $j \in \M'$ and $\bar{y_j} = 0$ if $j \notin \M'$. We call this a $\lbrace \frac{1}{2},1 \rbrace$-integral solution.
  \item Step 3: Round $\{ \bar{y_j} \}, \{ \bar{z_{ij}} \}$ to obtain an integer solution $\{ \hat{y_j} \}, \{ \hat{z_{ij}} \}$.
\end{itemize}

We introduce constraint~\ref{con4} in our linear program, and make a small modification to Step 2b as described in sections below.

\subsubsection{Step 1: Consolidating Demands}
Our first observation is that during the demand consolidation, it cannot be the case that all of the demand within a given $B(i, \gamma R_i)$ from constraint~\ref{con4} is moved arbitrarily far away.

\begin{lemma}
\label{lemma:demandMovedTo5Ri}
    Fix $i \in \M$. For each $j \in B(i, \gamma R_i)$ that had its demand moved to $j' \in \M'$, $d(i, j') \leq 9 \gamma R_i$.
\end{lemma}
\begin{proof}
    Let $j \in B(i, \gamma R_i)$. Let $j'$ be the location to which the demand for $j$ was moved in Step 1. Note that Step 1 is designed so that if demand at $j$ is moved to another point $j'$, then $\bar{C_{j'}} \leq \bar{C_j}$, and $c_{jj'} \leq 4\max(\bar{C_j},\bar{C_{j'}}) = 4 \bar{C_j}$. By constraint~\ref{con4}, we know that the demand of $j$ could be completely satisfied by centers fractionally opened inside of $B(i, \gamma R_i)$, so $\bar{C_j} \leq 2 \gamma R_i \implies c_{jj'} \leq 8 \gamma R_i$. Since $j \in B(i, \gamma R_i)$, it follows that $d(i, j') \leq 9 R_i$.
\end{proof}

\subsubsection{Step 2: Consolidating Centers}
Next, we argue that the consolidation of fractional centers in Step 2 approximately preserves constraint~\ref{con4}.

\begin{lemma}
\label{lemma:centerIn5Ri}
    After Step 2a, $\forall j \in \M$, $\sum_{j' \in B(j, 9 \gamma R_j)} y'_{j'} \geq 1$.
\end{lemma}
\begin{proof}
    Consider each location $j' \in B(j, \gamma R_j)$. Let $j''$ be the location in $\M'$ closest to $j'$. By Lemma~\ref{lemma:demandMovedTo5Ri}, there exists such a $j''$ with $d(j, j'') \leq 9 \gamma R_j$, so $j'' \in B(j, 9 \gamma R_j)$. Step 2a will move the fractional center at $j'$ to $j''$, so $y'_{j''} \geq \min(1, y_{j'} + y_{j''}) \geq y_{j'}$. Summing over all $j' \in B(j, \gamma R_j)$, we have $\sum_{j'' \in B(j, 9 \gamma R_j)} y'_{j''} \geq \sum_{j' \in B(j, \gamma R_j)} y_j' \geq 1$ by constraint~\ref{con4}.
\end{proof}

For our algorithm, we slightly change Step 2b to the following: Let $\M'' = \lbrace j\in \M': \; y'_j < 1 \rbrace$, $m'=|\M'|$ and $m''=|\M''|$. Sort the locations $j \in \M''$ in decreasing order of $t'_j d(s(j), j)$, where $s(j)$ is the location in $\M'$ closest to $j$ (other than $j$). Set $\bar{y_j} = 1$ for the first $2k - 2m' + m''$ locations in $\M''$ or if $j \in \M'\backslash \M''$, and $\bar{y_j} = \frac{1}{2}$ otherwise. In other words, the only difference from the original Step 2b in \cite{kMedianLP} is that points with integral $y'_j = 1$ will not participate in the sorting; we simply set $\bar{y_j} = 1$ for such points, and then perform the standard rounding on $M''$. \cite{kMedianLP} use the following statement in their proof, and we show that it still holds after our modification to Step 2b.

\begin{lemma}
\label{lemma:roundingCost}
    For any $\frac{1}{2}$-restricted solution $\{y'_j\}, \{z'_{ij}\}$, the modified Step 2b givs a $\lbrace \frac{1}{2},1 \rbrace$-integral solution with no greater cost.
\end{lemma}
\begin{proof}
    From Lemma 7 in \cite{kMedianLP}, the cost of the $\frac{1}{2}$-restricted solution $\{y'_j\}, \{z'_{ij}\}$ is
    \begin{flalign*}
        &&&&&\phantom{{}={}} \sum_{j \in \M'} t'_j \left( 1 - y'_j \right) d(s(j), j)&&\\
        &&&&&=\sum_{j \in \M''} t'_j \left( 1 - y'_j \right) d(s(j), j)&&\\
        &&&&&=\sum_{j \in \M''} t'_j d(s(j), j))-\sum_{j \in \M''} t'_j d(s(j), j) y'_j.&&
    \end{flalign*}
    where the second line follows because $y'_j = 1 \; \forall j\in \M' \backslash \M''$. Our algorithm in Step 2b maximizes $\sum_{j\in \M''} t'_j d(s(j), j) \bar{y_j}$ for the given set of $t'_j d(s(j), j)$, hence achieves a cost at most that of $\{y'_j\}, \{z'_{ij}\}$.
\end{proof}

\begin{lemma}
\label{lemma:twoHalfCenters}
    After Step 2b, $\forall j \in \M$, there is either at least one $j' \in B(j, 9 \gamma R_j)$ with $\bar{y_{j'}} = 1$ or at least two $j' \in B(j, 9 \gamma R_j)$ with $\bar{y_{j'}} \geq \frac{1}{2}$.
\end{lemma}
\begin{proof}
    Given Lemma~\ref{lemma:centerIn5Ri} and the constraints on $y'_i$ after Step 2a, there must be at least one $j' \in B(j, 9 \gamma R_j)$ with positive demand. If there is exactly one such $j'$, Lemma 3 is equivalent to $y'_{j'} = 1$ and Step 2b will ensure $\bar{y_{j'}} = 1$. If there are at least two such $j'$, all of them will have $\bar{y_{j'}} \geq \frac{1}{2}$ after Step 2b.
\end{proof}

\subsubsection{Step 3: Rounding an Integer Solution}
\cite{kMedianLP} gives two rounding schemes, and we use the first one that at most doubles the cost. The important observation is that any center with $\bar{y_j} = 1$ will be opened in the integral solution (that is, if $\bar{y_j} = 1$ then $\hat{y_j}=1$, and for any center with $\bar{y_j}=\frac{1}{2}$, either $j$ itself or another center in $\M'$ closest to $j$ will be opened. This allows us to complete our argument. 

\begin{lemma}
    For all $j \in \M$, $\sum_{j' \in B(j, \, 27 \gamma R_j)} \hat{y_{j'}} \geq 1$.
\end{lemma}
\begin{proof}
    By Lemma~\ref{lemma:twoHalfCenters}, there are two cases: either there is some $j' \in B(j, 9 \gamma R_j)$ with $\bar{y_{j'}} = 1$ or there are at least two $j' \in B(j, 9 \gamma R_j)$ with $\bar{y_{j'}} \geq \frac{1}{2}$. In the first case, $\hat{y_{j'}} = \bar{y_{j'}} = 1$, so the lemma statement clearly holds. In the second case, we are guaranteed that for each point $j'$, either we set $\hat{y_{j'}} = 1$ (in which case the Lemma holds) or we set $\hat{y_{j''}} = 1$ where $j''$ is the closest other at least partially open center. But since in this case there are \textit{two} points in $B(j, 9 \gamma R_j)$ partially open, their pairwise distance is at most $18 \gamma R_j$, so $d(j', j'') \leq 18 \gamma R_j \implies d(j, j'') \leq 27 \gamma R_j.$
\end{proof}

This concludes the proof of Lemma~\ref{lemma:rounding}. The fact that $\{ \hat{y_j} \}, \{ \hat{z_{ij}} \}$ is an 8-approximation of the objective follows immediately from the proof from \cite{kMedianLP} given Lemma~\ref{lemma:roundingCost}, as all other constraints are still satisfied. Finally, we note that the constant factor of 27 can be tightened to 13 in the special case where $\N = \M$. The argument is essentially the same. The crucial improvement comes from the guarantee that $\forall i \in \M$, there is demand at the \textit{center} of each $B(i, \gamma R_i)$. Tracking this demand throughout the rounding leads to the tightened result.

\section{Sampling for Linear-Time Implementations and Auditing}
\label{sec:audit}
In this section, we study proportionality under uniform random sampling (i.e., draw $|N|$ individuals i.i.d. from the uniform distribution on $\N$). In particular, we show that proportionality is well preserved under random sampling. This allows us to design efficient implementations of Algorithm~\ref{alg:ball} and Algorithm~\ref{alg:local}, and to introduce an efficient algorithm for auditing proportionality. 
We first present the general property and then demonstrate its various applications. 

\subsection{Proportionality Under Random Sampling}
For any $X \subseteq \M$ of size $k$ and center $y \in \M$, define $R(\N, X, y) = \{i \in \N : D_i(X) > \rho \cdot d(i,y) \}$. Note that solution $X$ is not $\rho$-proportional with respect to $\N$ if and only if there is some $y \in \M$ such that $\frac{|R(\N,X,y)|}{|\N|} \geq \frac{1}{k}$. 
A random sample approximately preserves this fraction for all solutions $X$ and deviating centers $y$. The important idea in the proof is that we take a union bound over all possible solutions and deviations, and there are only $k {m \choose k}$ such combinations. 

\begin{theorem}
\label{thm:sampling}
Given $\N$, $\M$ and parameter $\rho \ge 1$, fix parameters $\epsilon, \delta \in [0,1]$. Let $N \subseteq \N$ of size $\Omega\left(\frac{k^3}{\epsilon^2} \log \frac{m}{\delta}\right)$ be chosen uniformly at random. Then, with probability at least $1-\delta$, the following holds for all $(X,y)$:
$$ \left| \frac{| R(N,X,y) |}{|N|} - \frac{| R(\N,X,y)  |}{|\N|} \right| \le \frac{\epsilon}{k}$$	
\end{theorem}
\begin{proof}
	Recall that $N$ is a random sample of $\N$. Hoeffding's inequality implies that for any fixed $(X,y)$, a sample of size $|N| = O\left( \frac{1}{\hat{\epsilon}^2} \log \frac{1}{\hat{\delta}} \right)$ is sufficient to achieve
$$  \left| \frac{| R(N, X,y) |}{|N|} - \frac{| R(\N, X, y)  |}{|\N|} \right| \le \hat{\epsilon}$$ 
with probability at least $1-\hat{\delta}$. Note that there are $q = m {m \choose k}$ possible choices of $(X,y)$ over which we take the union bound. Setting $\delta = \frac{\hat{\delta}}{q}$, and $\epsilon = \frac{\hat{\epsilon}}{k}$ is sufficient for the union bound to yield the theorem statement.
\end{proof}

In order to apply the above theorem, we say that a solution $X$ is $\rho$-proportional to $(1+\epsilon)$-deviations if for all $y \in \M$ and for all $S \subseteq \N$ where $|S| \geq (1+\epsilon) \frac{n}{k}$, there exists some $i \in S$ such that $\rho \cdot d(i, y) \geq D_i(X)$. Note that if $X$ is $\rho$-proportional to 1-deviations, it is simply $\rho$-proportional. We immediately have the following: 

\begin{corollary}
Let $N \subset \N$ be a uniform random sample of size $|N| = \Omega \left(\frac{k^3}{\epsilon^2} \ln \frac{m}{\delta}\right)$. Suppose $X \subseteq M$ with $|X| = k$ is $\rho$-proportional with respect to $N$. Then with probability at least $1-\delta$, $X$ is $\rho$-proportional to $(1+\epsilon)$-deviations with respect to $\N$.
\end{corollary}

\subsection{Linear Time Implementation}
We now consider how to take advantage of Theorem~\ref{thm:sampling} to optimize Algorithm~\ref{alg:ball} and Algorithm~\ref{alg:local}. First, note that Algorithm~\ref{alg:ball} takes $\tilde{O}(m n)$ time, which is quadratic in input size. 
A corollary of Theorem~\ref{thm:sampling} is that we can approximately implement Algorithm~\ref{alg:ball} in nearly linear time, comparable to the running time of the standard $k$-means heuristic.

\begin{corollary} 
\label{cor:fastGreedy}
Algorithm~\ref{alg:ball}, when run on $\M$ and a random sample $N \subseteq \N$ of size $|N| = \tilde{\Theta}\left(\frac{k^3}{\epsilon^2}\right)$, provides a solution that is $(1+\sqrt{2})$-proportional to $(1+\epsilon)$-deviations with high probability in $\tilde{O}\left( \frac{k^3}{\epsilon^2} m \right)$ time.
\end{corollary}

We also get a substantial speedup for our Local Capture algorithm. Recall that Local Capture (Algorithm~\ref{alg:local}) is an iterative algorithm that takes a target value of $\rho$ as a parameter, and if it converges, returns a $\rho$-proportional clustering. Without sampling, each iteration of Algorithm~\ref{alg:local} takes $\tilde{O}(m n^2)$ time. Another corollary of Theorem~\ref{thm:sampling} is that it is sufficient to run the Local Capture on a random sample of $k^3/\epsilon^2$ out of the $n$ points in $\N$ in order to search for a clustering that is $\rho$-proportional with respect to $(1+\epsilon)$-deviations.

\subsection{Efficient Auditing}
Alternatively, one might still want to run a non-proportional clustering algorithm, and ask whether the solution produced happens to be proportional. We call this the \textit{Audit Problem}. Given $\N, \M,$ and $X \subset \N$ with $|X| \leq k$, find the minimum value of $\rho$ such that $X$ is $\rho$-proportional. 
It is not too hard to see that one can solve the Audit Problem exactly in $O((k+m)n)$ time by computing for each $y \in \M$, the quantity $\rho_y$, the $\lceil \frac{n}{k} \rceil$ largest value of $\frac{D_i(X)}{d(i, y)}$. We subsequently find the $y$ that maximizes $\rho_y$. Again, this takes quadratic time, which can be worse than the time taken to find the clustering itself.

Consider a slightly relaxed $(\epsilon, \delta)$-Audit Problem where we are asked to find the minimum value of $\rho$ such that $X$ is $\rho$-proportional to $(1+\epsilon)$-deviations with probability at least $1-\delta$. 
This problem can be efficiently solved by using a random sample $N \subseteq \N$ of points to conduct the audit.

\begin{corollary}
\label{cor:fastAudit}
The $(\epsilon, \delta)$-Audit Problem can be solved in $\tilde{O}\left( \left(k + m \right)\frac{k^3}{\epsilon^2} \right)$ time.
\end{corollary}

\section{Implementations and Empirical Results}
\label{sec:experiments}
In this section, we study proportionality on real data taken from the UCI Machine Learning Repository~\cite{UCI}. We consider three qualitatively different data sets used for clustering: Iris, Diabetes, and KDD. For each data set, we only have a single set of points given as input, so we take $\N = \M$ to be the set of all points in the data set. We use the standard Euclidean L2 distance.

\textbf{Iris.} This data set contains information about the petal dimensions of three different species of iris flowers. There are 50 samples of each species.

\textbf{Diabetes.} The Pima Indians Diabetes data set contains information about 768 diabetes patients, recording features like glucose, blood pressure, age and skin thickness. 

\textbf{KDD.} The KDD cup 1999 data set contains information about sequences of TCP packets. Each packet is classified as normal or one of twenty-two types of intrusions. Of these 23 classes, normal, ``neptune'', and ``smurf'' account for 98.3$\%$ of the data. The data set contains 18 million samples; we work with a subsample of 100,000 points.\footnote{We run $k$-means++ on this entire 100,000 point sample. For efficiency, we run our Local Capture algorithm by further sampling 5,000 points uniformly at random to treat as $\N$ and sampling 400 points via the $k$-means++ initialization to treat as $\M$. For the sake of a fair comparison, we generate a different sample of 400 centers using the $k$-means++ initialization that we use to determine the value of $\rho$ we report for both Local Capture and the $k$-means++ algorithm. The $k$-means objective is measured on the original 100,000 points for both algorithms.}

\subsection{Proportionality and $k$-means Objective Tradeoff}

We compare Greedy Capture (Algorithm~\ref{alg:ball}) and Local Capture (Algorithm~\ref{alg:local}) with the $k$-means++ algorithm (Lloyd's algorithm for $k$-means minimization with the $k$-means++ initialization~\cite{kmeans++}) for a range of values of $k$. For the Iris data set, Local Capture  and $k$-means++ always find an exact proportional solution (Figure~\ref{fig:IrisRho}), and have comparable $k$-means objectives (Figure~\ref{fig:IriskMeans}). The Iris data set is very simple with three natural clusters, and validates the intuition that proportionality and the $k$-means objective are not always opposed.

The Diabetes data set is larger and more complex. As shown in Figure~\ref{fig:DiabetesRho}, $k$-means++ no longer always finds an exact proportional solution. Local Capture always finds a better than 1.01-proportional solution. As shown in Figure~\ref{fig:DiabeteskMeans}, the $k$-means objectives of the solutions are separated, although generally on the same order of magnitude.

\begin{figure*}[t!]
    \centering
    \begin{subfigure}[t]{0.33\linewidth}
        \includegraphics[width=\linewidth]{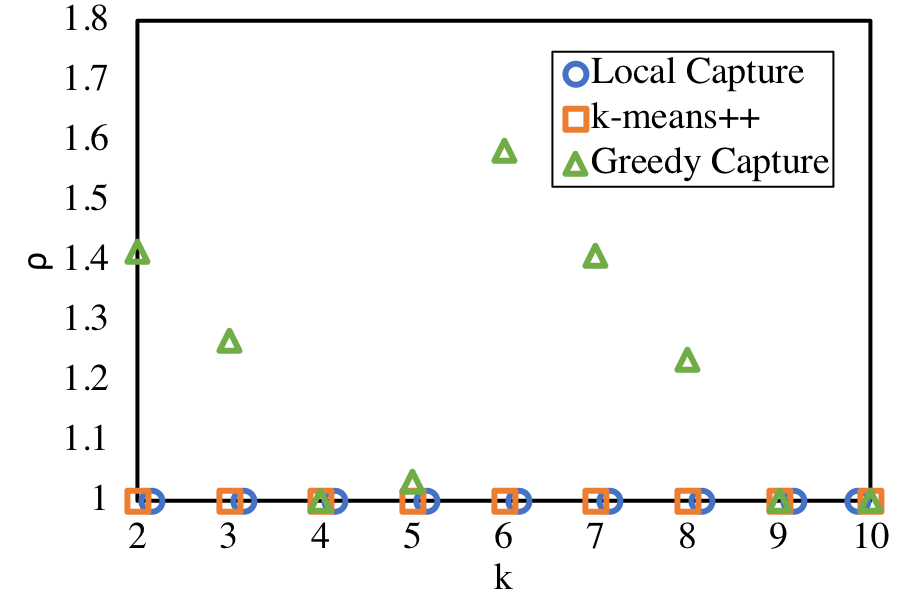}
        \caption{Iris}
        \label{fig:IrisRho}
    \end{subfigure}%
    ~
    \begin{subfigure}[t]{.33\linewidth}
        \includegraphics[width=\linewidth]{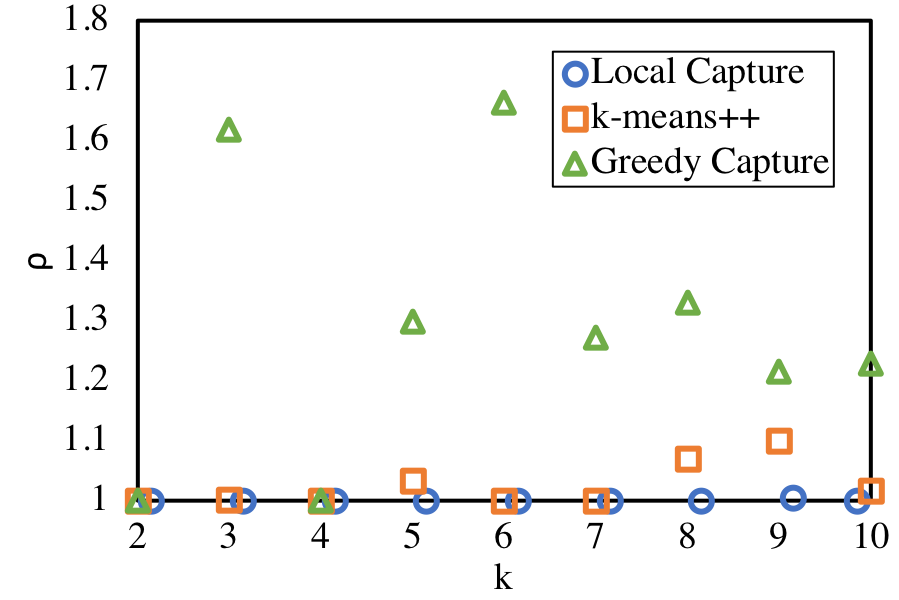}
        \caption{Diabetes}
        \label{fig:DiabetesRho}
    \end{subfigure}%
    ~
    	 \begin{subfigure}[t]{.33\linewidth}
        \includegraphics[width=\linewidth]{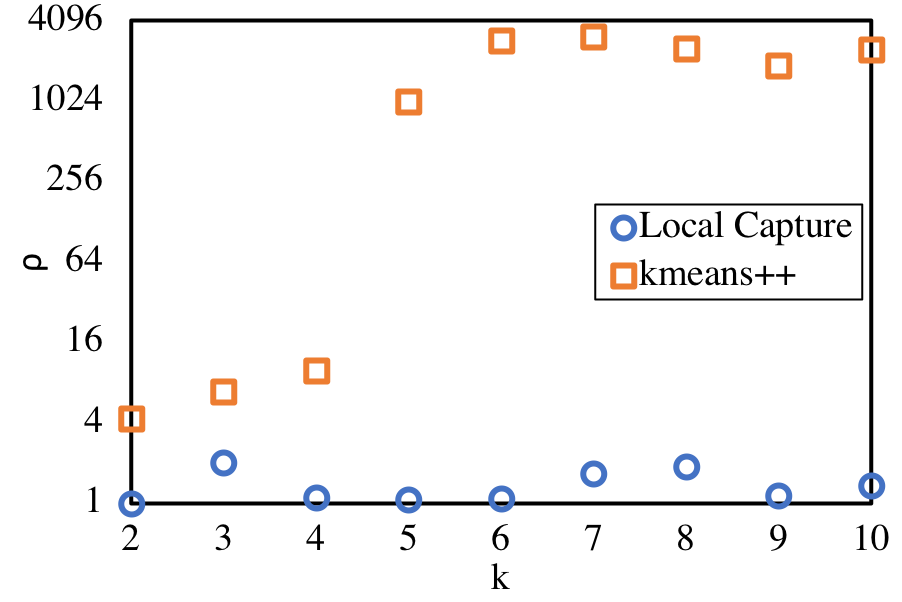}
        \caption{KDD, geometric scale}
        \label{fig:kddRho}
    \end{subfigure}%
    \caption{Minimum $\rho$ such that the solution is $\rho$-proportional}.
    \label{fig:Rho}
\end{figure*}
    
\begin{figure*}[t!]
    \centering
    \begin{subfigure}[t]{0.33\linewidth}
        \includegraphics[width=\linewidth]{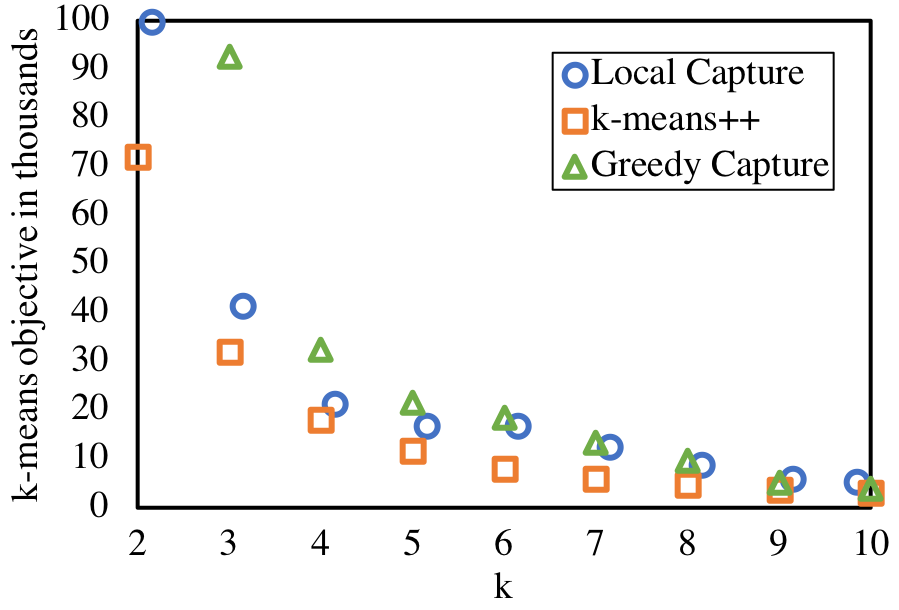}
        \caption{Iris}
        \label{fig:IriskMeans}
    \end{subfigure}%
    ~
    \begin{subfigure}[t]{0.33\linewidth}
        \includegraphics[width=\linewidth]{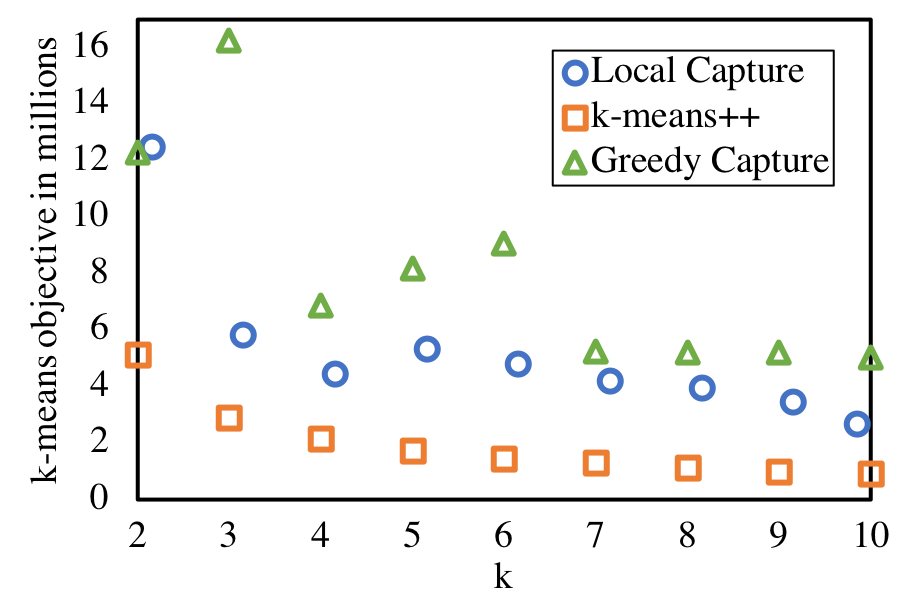}
        \caption{Diabetes}
        \label{fig:DiabeteskMeans}
    \end{subfigure}%
    ~ 
    \begin{subfigure}[t]{0.33\linewidth}
        \includegraphics[width=\linewidth]{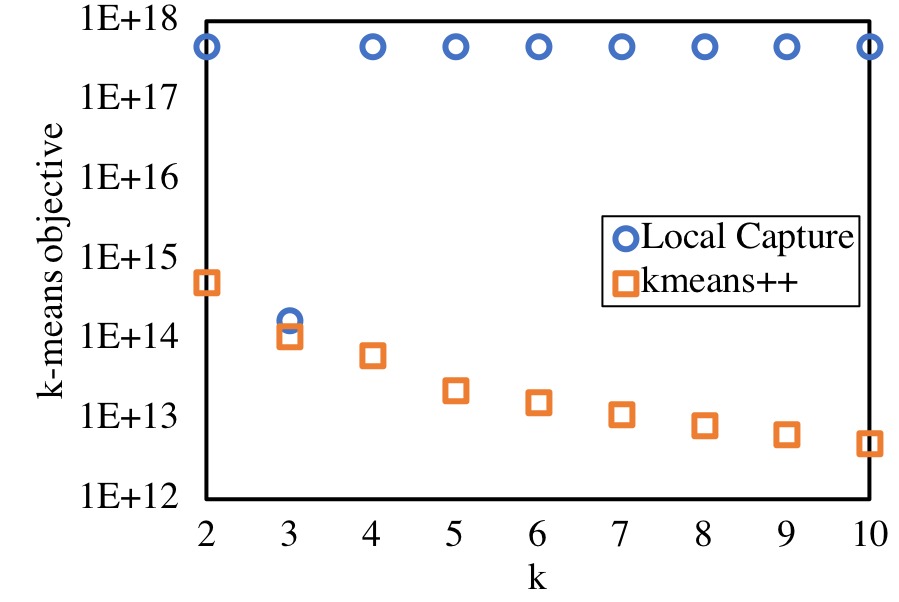}
        \caption{KDD, geometric scale}
        \label{fig:kddObjective}
    \end{subfigure}
    \caption{$k$-means objective}
\end{figure*}

For the KDD data set, proportionality and the $k$-means object appear to be in conflict. Greedy Capture's performance is comparable to Local Capture on KDD, so we omit it for clarity. In Figures~\ref{fig:kddRho} and~\ref{fig:kddObjective}, note that the gap between $\rho$ and the $k$-means objective for the $k$-means++ and Local Capture algorithms is between three and four orders of magnitude. We suspect this is due to the presence of significant outliers in the KDD data set. This is in keeping with the theoretical impossibility of simultaneously approximating the optima on both objectives, and demonstrates that this tension arises in practice as well as theory.


\subsection{Proportionality and Low $k$-means Objective}
Note that if one is allowed to use $2k$ centers when $k$ is given as input, one can trivially achieve the proportionality of Local Capture and the $k$-means objective of the $k$-means++ algorithm by taking the union of the two solutions. Thinking in this way leads to a different way of quantifying the tradeoff between proportionality and the $k$-means objective: Given an approximately proportional solution, how many \textit{extra} centers are necessary to get comparable $k$-means objective as the $k$-means++ algorithm? For a given data set, the answer is a value between 0 and $k$, where larger numbers indicate more incompatibility, and lower numbers indicate less incompatibility.  

To answer this question, we compute the union of centers found by Local Capture and the $k$-means++ algorithm. We then greedily remove centers as long as doing so does not increase the minimum $\rho$ such that the solution is $\rho$-proportional (defined on $k$, not $2k$) by more than a multiplicative factor of $\alpha$, and does not increase the $k$-means objective by more than a multiplicative factor $\beta$.

On the KDD dataset, we set $\alpha = 1.2$ and $\beta = 1.5$, so the proportionality of the result is within 1.2 of Local Capture in Figure~\ref{fig:kddRho}, and the $k$-means objective is within 1.5 of $k$-means++ in Figure~\ref{fig:kddObjective}.  We observe that this heuristic uses at most $3$ extra centers for any $k \le 10$. So while there is real tension between proportionality and the $k$-means objective, this tension is still not maximal. In the worst case, one might need to add $k$ centers to a proportional solution to compete with the $k$-means objective of the $k$-means++ algorithm, but in practice we find that we need at most $3$ for $k \le 10$.

\section{Conclusion and Open Directions}
\label{sec:open}
We have introduced proportionality as a fair solution concept for centroid clustering. Although exact proportional solutions may not exist, we gave efficient algorithms for computing approximate proportional solutions, and considered constrained optimization and sampling for further applications. Finally, we studied proportionality on real data and observed a data dependent tradeoff between proportionality and the $k$-means objective. While this tradeoff is in some sense a negative result, it also demonstrates that proportionality as a fairness guarantee matters in the sense that it meaningfully constrains the space of solutions. 

We have shown that $\rho$-proportional solutions need not exist for $\rho < 2$, and always exist for $\rho \geq 1+\sqrt{2}$. Closing this approximability gap is one outstanding question. Another is whether there is a more efficient and easily interpretable algorithm for optimizing total cost subject to proportionality, as our approach in Section~\ref{sec:lp} requires solving a linear program on the entire data set. We would ideally like a more efficient and easily interpretable primal-dual or local search type algorithm. More generally, what other fair solution concepts for clustering should be considered alongside proportionality, and can we characterize their relative advantages and disadvantages? Finally, can the idea of proportionality as a group fairness concept be adapted for supervised learning tasks like classification and regression? 

\section*{Acknowledgements}
Brandon Fain is supported by NSF grants CCF-1408784 and IIS-1447554. Kamesh Munagala is supported by NSF grants CCF-1408784, CCF-1637397, and IIS-1447554; and research awards from Adobe and Facebook.

\bibliographystyle{abbrv}
\bibliography{refs.bib}



\end{document}